\documentclass{article}
\usepackage{times}
\usepackage{graphicx}
\usepackage{subfigure} 
\usepackage{algorithm}
\usepackage{algorithmic}
\usepackage{epsfig}
\usepackage{setspace}
\usepackage{multirow}
\usepackage{epstopdf} 
\usepackage{tabularx}
\usepackage{enumitem}
\setlist{nolistsep}
\usepackage{amsmath}
\usepackage{indentfirst}
\usepackage{amssymb}
\usepackage{amsfonts}
\usepackage{amsthm}
\usepackage{subfigure}
\usepackage{comment}
\usepackage{color}
\usepackage{verbatim}
\usepackage[]{graphicx}
\newcommand{\argmin}{\operatornamewithlimits{argmin}}
\newcommand{\circR}{\operatornamewithlimits{circ}}

\usepackage{bm}
\usepackage{authblk}
\newtheorem{theorem}{Theorem}
\newtheorem{proposition}{Proposition}

\usepackage{bbm}

\newcommand{\bR}{\mathbf{R}}
\newcommand{\br}{\mathbf{r}}
\newcommand{\bT}{\mathbf{\Theta}}
\newcommand{\bx}{\mathbf{x}}
\newcommand{\by}{\mathbf{y}}
\newcommand{\bz}{\mathbf{z}}
\newcommand{\bw}{\mathbf{w}}
\newcommand{\bt}{\bm{\theta}}

\graphicspath{{Figure/}}
\renewcommand{\[}{\begin{eqnarray*}}
\renewcommand{\]}{\end{eqnarray*}}

\usepackage[top=1.5in, bottom=1.5in, left=1in, right=1in]{geometry}

\begin{document} 
\title{Compact Nonlinear Maps and Circulant Extensions}
\date{}

\author[1]{Felix X. Yu\thanks{yuxinnan@ee.columbia.edu}}
\author[2]{Sanjiv Kumar}
\author[2]{Henry Rowley}
\author[1]{Shih-Fu Chang}
\affil[1]{Department of Electrical Engineering, Columbia University}
\affil[2]{Google Research}
\maketitle

\begin{abstract}
Kernel approximation via nonlinear random feature maps is widely used in speeding up kernel machines. There are two main challenges for the conventional kernel approximation methods. First, before performing kernel approximation, a good kernel has to be chosen. Picking a good kernel is a very challenging problem in itself. 
Second, high-dimensional maps are often required in order to achieve good performance. This leads to high computational cost in both generating the nonlinear maps, and in the subsequent learning and prediction process. 
In this work, we propose to optimize the nonlinear maps directly with respect to the classification objective in a data-dependent fashion. The proposed approach achieves kernel approximation and kernel learning in a joint framework. This leads to much more compact maps without hurting the performance. As a by-product, the same framework can also be used to achieve more compact kernel maps to approximate a known kernel. We also introduce  Circulant Nonlinear Maps, which uses a circulant-structured projection matrix to speed up the nonlinear maps for high-dimensional data.
\end{abstract}

\section{Introduction}

Kernel methods such as the Support Vector Machines (SVMs) \cite{cortes1995support} are widely used in machine learning to provide nonlinear decision function. The kernel methods use a positive-definite kernel function $K$ to induce an implicit nonlinear map $\phi$ such that $K(\bx, \by) = \langle \phi(\bx), \phi(\by) \rangle$, $\bx, \by \in \mathbb{R}^d$. This implicit feature space could potentially be an infinite dimensional space. Fortunately, kernel methods allow one to utilize the power of these rich feature spaces without explicitly working in such high dimensions.
Despite their popularity, the kernel machines come with high computational cost due to the fact that at the training time it is necessary to compute a large kernel matrix of size $N \times N$ where $N$ is the number of training points. Hence the overall training complexity varies from $O(N^2)$ to $O(N^3)$, which is prohibitive when training with millions of samples. Testing also tends to be slow due to the linear growth in the number of support vectors with training data, leading to $O(Nd)$ complexity for $d$-dimensional vectors.

On the other hand, linear SVMs are appealing for large-scale applications since they can be trained in $O(N)$ time \cite{joachims2006training, fan2008liblinear, shalev2011pegasos} and applied in $O(d)$ time, independent of $N$. Hence, if the input data can be mapped nonlinearly into a compact feature space explicitly, one can utilize fast training and testing of linear methods while still preserving the expressive power of kernel methods. 

Following this reasoning, kernel approximation via explicit nonlinear maps has become a popular strategy for speeding up kernel machines \cite{rahimi2007random}. Formally, given a kernel $K(\bx, \by)$, kernel approximation aims at finding a nonlinear map $Z(\cdot)$, such that
\begin{equation*}
K(\bx, \by) \approx Z(\bx)^T Z(\by)
\end{equation*}

However, there are two main issues with the existing nonlinear mapping  methods. Before the kernel approximation, a ``good'' kernel has to be chosen. Choosing a good kernel is perhaps an even more challenging problem than approximating a known kernel. In addition, the existing methods are designed to approximate the kernel in the whole space independent on the data. As a result, the feature mapping often needs to be high-dimensional in order to achieve low kernel approximation error. 

In this work, we propose an alternative formulation that optimizes the nonlinear maps directly in a data-dependent fashion. Specifically, we adopt the Random Fourier Feature  framework  \cite{rahimi2007random} for approximating positive definite shift-invariant kernels. Instead of generating the parameter of the nonlinear map randomly from a distribution, we learn the parameters by minimizing the classification loss based on the training data (Section \ref{sec:cnm}). The proposed method can be seen as approximating an ``optimal kernel'' for the classification task. The method results in significantly more compact maps with very competitive classification performance. As a by-product, the same framework can also be used to achieve compact kernel approximation, if the goal is to approximate some predefined kernels (Section \ref{sec:kernel_app}). The proposed compact nonlinear maps are fast to learn, and compare favorably to the baselines. 
In addition, to make the method scalable for very high-dimensional data, we propose to use circulant structured projection matrices in the nonlinear maps (Section \ref{sec:circulant_ker}). 
This further improves the computational complexity from $\mathcal{O}(kd)$ to $\mathcal{O}(k \log d) $ and the space complexity from $\mathcal{O}(k d)$ to $\mathcal{O}(k)$, where $k$ is the number of nonlinear maps, and $d$ is the input dimensionality. 

\section{Related Works}

\textbf{Kernel Approximation.}
Following the seminal work on explicit nonlinear feature maps for approximating positive definite shift-invariant kernels~\cite{rahimi2007random}, nonlinear mapping techniques have been  proposed to approximate other forms of kernels such as the polynomial kernel \cite{kar2012random, pham2013fast}, generalized RBF kernels \cite{sreekanth2010generalized}, intersection kernels \cite{maji2009max}, additive kernels \cite{vedaldi2012efficient}, skewed multiplicative histogram kernels \cite{li2010random}, and semigroup kernel \cite{yang2014random}. Techniques have also been proposed to improve the speed and compactness of kernel approximations by using structured projections \cite{le2013fastfood}, better quasi Monte Carlo  sampling \cite{yang2014quasi}, binary code \cite{zhang2013new, mu2010weakly}, and dimensionality reduction \cite{hamid2014compact}. Our method in this paper is built upon the Random Fourier Feature \cite{rahimi2007random} for approximating shift-invariant kernel, a widely used kernel type in machine learning.
Besides explicit nonlinear maps, kernel approximation can also be 
achieved using sampling-based low-rank approximations of the kernel matrices such as the Nystrom method \cite{williams2001using, drineas2005nystrom, kumar2009sampling}. In order for these approximations to work well, the eigenspectrum of the kernel matrix should have a large gap \cite{yang2012nystrom}.

\textbf{Kernel Learning.}
There have been significant efforts in learning a good kernel for the kernel machines. Works have been proposed to optimize the hyperparameters of a kernel function \cite{chapelle2002choosing, keerthi2006efficient}, and  finding the best way of combining multiple kernels, \emph{i.e.}, Multiple Kernel Learning (MKL) \cite{bach2004multiple, argyriou2006dc, gehler2008infinite, cortes2009learning}. A summary of MKL can be found in \cite{gonen2011multiple}.
Related to our work, \cite{buazuavan2012fourier, ghiasi2010learning} propose to optimize shift-invariant kernels. Different from the above, the proposed approach can be seen as learning an optimal kernel by directly optimizing its nonlinear maps. Therefore, it is a joint kernel approximation and kernel learning. 

\textbf{Fast Nonlinear Models.}
Besides kernel approximation, there have been other types of works aiming at speeding up kernel machine \cite{bottou2007large}.
Such techniques include decomposition methods \cite{osuna1997support, CC01a}, sparsifying kernels \cite{achlioptas2001fast}, limiting the number of support vectors \cite{keerthi2006building, pavlov2000towards}, and low-rank approximations \cite{fine2002efficient, bach2005predictive}. 
None of the above methods can be scaled to truly large-scale data.
Another alternative is to consider the local structure of the data to train and apply the kernel machines locally \cite{ladicky2011locally, hsieh2014divide, jose2013local, hsieh2014fast}. However, partitioning becomes unreliable in high-dimensional data.
Our work is also related to shallow neural networks as we will discuss in later part of this paper. 

\section{Random Fourier Features: A Review}

We begin by reviewing the Random Fourier Feature method~\cite{rahimi2007random}, which is widely used in approximating positive-definite shift-invariant kernels. 
A kernel $K$ is shift-invariant, if $K(\bx,\by) = K(\bz)$ where $\bz = \bx - \by$. For a function $K(\bz)$ which is positive definite on $\mathbb{R}^d$, it guarantees that the Fourier transform of $K(\bz)$, 
\begin{equation}
\mathcal{K}({\bt}) = \frac{1}{(2\pi)^{d/2}}\int d^d {\bf z}\, K({\bf z})\, e^{i  {\bt}^T{\bf z}}\,,
\end{equation}
admits an interpretation as a probability distribution. This fact follows from Bochner's celebrated characterization of positive definite functions,

\begin{theorem}
\label{thm:Bochner}
\cite{bochner1955harmonic} A function $K\in C(\mathbb{R}^d)$ is positive definite on $\mathbb{R}^d$ if and only if it is the Fourier transform of a finite non-negative Borel measure on $\mathbb{R}^d$.
\end{theorem}
\vspace{-0.1cm}
A consequence of Bochner's theorem is that the inverse Fourier transform of $\mathcal{K}({\bt})$, \emph{i.e.}, $K(\bz)$, can be interpreted as the computation of an expectation, \emph{i.e.},
\begin{align}
\label{eqn:Kz_sample}
 &K({\bf z}) = \frac{1}{(2\pi)^{d/2}}\int d^d {\bt}\, \mathcal{K}({\bt})\, e^{-i  {\bt}^T{\bf z}}\\
=&  E_{{\bt} \sim p({\bt})}\, e^{-i {\bt}^T (\bf{x-y})} \nonumber\\
=& 2\, E_{\substack{\bt\sim p(\bt) \\ b\sim U(0,2\pi)}} \big[\cos({\bt}^T {\bf x} + b) \cos({\bt}^T{\bf y}+ b) \big]\ \nonumber,
\end{align}
where $p({\bt}) = (2\pi)^{-d/2} \mathcal{K}({\bt})$ and $U(0, 2\pi)$ is the uniform distribution on $[0, 2\pi)$. If the above expectation is approximated using Monte Carlo with $k$ random samples $\{\bt_i, b_i\}_{i=1}^k$,  then $K(\bx, \by) \approx \langle Z(\bx), Z(\by) \rangle$ with 
\begin{equation}
Z(\bx) = \sqrt{2 / k} \left[ \cos(\bt_1^T \bx + b_1), ..., \cos(\bt_k^T \bx + b_k) \right]^T \,.
\end{equation}

Such Random Fourier Features have been used to approximate different types of positive definite shift-invariant kernels, including the Gaussian kernel, the Laplacian kernel, and the Cauchy kernel \cite{rahimi2007random}. 
Despite the popularity and success of Random Fourier Feature, the notable issues for all kernel approximation methods are that: 
\begin{itemize}[leftmargin=*]
\item Before performing the kernel approximation, a known kernel has to be chosen. This is a very challenging task.  As a matter of fact, the classification performance is influenced by both the quality of the kernel, and the error in approximating it. Therefore, better kernel approximation in itself may not lead to better classification performance. 
\item The Monte-Carlo sampling technique tries to approximate the kernel for \emph{any} pair of points in the entire input space without considering the data distribution. This usually leads to very high-dimensional maps in order to achieve low kernel approximation error everywhere. 
\end{itemize}

In this work, we follow the Random Fourier Feature framework. Instead of sampling the kernel approximation parameters $\bt_i$ and $b_i$ from a probability distribution to approximate a known kernel, we propose to optimize them directly with respect to the classification objective. This leads to very compact maps as well as higher classification accuracy. 

\section{The Compact Nonlinear Map (CNM)}
\label{sec:cnm}
\subsection{The Framework}
Consider the following feature maps, and the resulted kernel based on the Random Fourier Features proposed in~\cite{rahimi2007random}\footnote{For simplicity, we do not consider the bias term which can be added implicitly by augmenting the  dimension to the feature $\bx$.}:
\begin{equation}
\label{eq:general}
\hat{K}_{\bT}(\bx, \by) = Z(\bx)^T Z(\by), \quad Z_i(\bx) = \sqrt{2/k}\cos(\bt_i^T \bx), \quad i = 1, ..., k.
\end{equation}
By representing $\bT = [\bt_1, \cdots, \bt_k]$, we can write 
$Z(\bx) = \cos(\bT^T \bx)$, where $\cos(\cdot)$ is the element-wise consine function.

\begin{proposition}
For any $\bT$, the kernel function $\hat{K}$, defined as $\hat{K}_{\bT}(\bx, \by) = Z(\bx)^T Z(\by)$, is a positive-definite shift-invariant kernel.
\end{proposition}
\begin{proof}
The shift-invariance follows from the fact that, for any $x, y \in  \mathbb{R}$
\begin{equation*}
\cos(x) \cos(y) = \frac{ \cos(x-y) - \sin(x-y)}{2}, \quad \text{a function of } x - y.
\end{equation*}
The positive definiteness follows from a direct computation and the definition. 
\end{proof}

In addition, it has been shown in the Bochner's theorem that such a cosine map can be used to approximate \emph{any} positive shift-invariant kernels. Therefore, if we optimize the ``kernel approximation'' parameters directly, it can be seen as approximating an optimal positive definite shift-invariant kernel for the task. In this work we consider the task of binary classification using SVM. The proposed approach can be easily extended to other scenarios such as multi-class classification and regression. 

Suppose we have $N$ samples with +1/-1 labels as training data $(\bx_1, y_1), ..., (\bx_N, y_N)$. The Compact Nonlinear Maps (CNM) jointly optimize the nonlinear map parameters $\bT$ and the linear classifier $\bw$ in a data-dependent fashion. 

\begin{equation}
\argmin_{\bw, \bT} \frac{\lambda}{2} \bw^T \bw + \frac{1}{N} \sum_{i=1}^N 
L \left( y_i, \bw^T Z(\bx_i) \right)
\label{eq:cnm_obj}
\end{equation}

In this paper, we use the hinge loss as the loss function: $L(y_i, \bw^T Z(x_i)) = \max(0, 1 - y_i \bw^T Z(x_i) )$.

\begin{algorithm}[t]
\begin{algorithmic}[1]
\STATE INPUT: initialized $\bw$, $||\bw|| < 1/\sqrt{\lambda}$.
\STATE OUTPUT: updated $\bw$.
\FOR{$t = 1$ to $T_1$}
\STATE Sample $M$ points to get $\mathcal{A}$, and compute the gradient $\nabla_{\bw}$.
\STATE $\bw \leftarrow \bw - (1 / \lambda t) \nabla_{\bw}$.
\STATE $\bw \leftarrow \min\left\{ 1, 1/ \lambda ||\bw|| \right\} \bw$.
\ENDFOR
\end{algorithmic}
\caption{Optimizing $\bw$ with fixed $\bT$}
\label{alg:pegasos}
\end{algorithm}

\subsection{The Alternating Minimization}
\label{sec:alt_min}

Optimizing Equation \ref{eq:cnm_obj} is a challenging task. A large number of parameters need to be optimized, and the problem is nonconvex. In this work, we propose to find a local solution of the optimization problem with  Stochastic Gradient Descent (SGD) in an alternating fashion. 

\textbf{For a fixed $\bT$}, the optimization of $\bw$ is simply
 the traditional linear SVM learning problem. 
 
\begin{equation}
\argmin_{\bw} \frac{\lambda}{2} \bw^T \bw + \frac{1}{N} \sum_{i=1}^N 
L \left( y_i, \bw^T Z(\bx_i) \right).
\label{eq:obj}
\end{equation}
 
We use the Pegasos procedure \cite{shalev2011pegasos} to perform SGD. In each step, we sample a small set of data points $\mathcal{A}$. The data points with non-zero loss is denotes as $\mathcal{A}_{+}$.
Therefore, the gradient can be written as 
\begin{equation}
\nabla_{\bw} = \lambda \bw - \frac{1} {|\mathcal{A}| }\sum_{(\bx, y) \in \mathcal{A}_{+}} y \cos(\bT^T \bx).
\end{equation}

Each step of the Pegasos procedure consists of gradient descent and a projection step. The process is summarized in Algorithm \ref{alg:pegasos}.

\textbf{For a fixed $\bw$}, optimizing $\bT$ becomes
\begin{equation}
\argmin_{\bw} \frac{1}{N} \sum_{i=1}^N 
L \left( y_i, \bw^T Z(\bx_i) \right).
\label{eq:obj_w}
\end{equation}

We also preform SGD with sampled mini-batches. Let the set of sampled data points be $\mathcal{A}$, the gradient can be written as

\begin{equation}
\nabla_{\bt_i} = \frac{w_i}{|\mathcal{A}|} \sum_{(\bx, y) \in \mathcal{A}_+} y \sin (\bt_i^T \bx) \bx,
\end{equation}

where $\mathcal{A}_+$ is the set of samples with non-zero loss, and $w_i$ is the $i$-th element of $\bw$.

\textbf{The overall algorithm} is shown in Algorithm \ref{alg:seq}. 
The sampled gradient descent steps are repeated to optimize $\bw$ and $\bT$ alternatively. We use a $\bT$ obtained from sampling the Gaussian distribution (same as Random Fourier Feature) as initialization.

\begin{algorithm}[t]
\begin{algorithmic}[1]
\STATE INPUT: initialized $\bT$.
\STATE OUTPUT: updated $\bT$.
\FOR{$t = 1$ to $T_2$}
\STATE Sample $M$ points to get $\mathcal{A}$, and compute the gradient $\nabla_{\bT}$.
\STATE $\bT \leftarrow \bT - (1 / \lambda t) \nabla_{\bT}$.
\ENDFOR
\end{algorithmic}
\caption{Optimizing $\bT$ with fixed $\bw$}
\label{alg:theta}
\end{algorithm}

\begin{algorithm}[t]
\begin{algorithmic}[1]
\STATE Initialize $\bT$ as the Random Fourier Feature.
\STATE Choose $\bw$ such that $||\bw|| < 1/\sqrt{\lambda}$.
\FOR{iter$=1$ to $T$}
\STATE Perform $T_1$ SGD (Pegasos \cite{shalev2011pegasos}) steps to optimize $\bw$, shown in Algorithm \ref{alg:pegasos}.
\STATE Perform $T_2$ SGD steps with to optimize $\bT$, shown in Algorithm \ref{alg:theta}.
\ENDFOR
\end{algorithmic}
\caption{The Compact Nonlinear Map (CNM)}
\label{alg:seq}
\end{algorithm}

\section{CNM for Kernel Approximation}
\label{sec:kernel_app}

In the previous section, we presented the Compact Nonlinear Maps (CNM) optimized to achieve low classification error. This framework can also be used to achieve compact kernel approximation. 
The idea is to optimize with respect to kernel approximation error. For example, given a kernel function $K$, we can minimize $\bT$ in terms of the MSE on the training data:
\begin{align}
\argmin_{\bT} \sum_{i=1}^N \sum_{j=1}^N \left(K(\bx_i, \bx_j) - Z(\bx_i)^T Z(\bx_j)\right)^2.
\end{align}

This can be used to achieve more compact kernel approximation by considering the data under consideration. Note that the ultimate goal of a  nonlinear map is to improve the classification performance -- therefore this section should be viewed as a by-product of the proposed method. 

For the optimization, we can also perform SGD similar to the former section. 
Let $\mathcal{A}$ be the set of random samples, we only need to compute the gradient in terms of $\bT$: 

\begin{align}
\nabla_{\bt_i}  = \frac{8}{k} 
\sum_{\bx, \bx' \in \mathcal{A}} 
\left( \mathcal{K}(\bx, \bx') 
 - \frac{2}{k} \cos(\bT^T \bx)^T \cos(\bt^T \bx')
 \right)  \sin(\bt_i^T \bx) \cos(\bt_i^T \bx') \bx_i.
\end{align}

\section{Discussions}

We presented Compact Nonlinear Maps (CNM) with an alternating optimization algorithm for the task of binary classification. CNM can be easily adapted to other tasks such as regression and multi-class classification. The only difference is that the gradient computation of the algorithm need to be changed. We provide below a brief discussion regarding adding regularization, and the relationship of CNM to neural networks. 

\subsection{Regularization}
One interesting fact is that, the $\cos$ function has an infinite VC dimension. In the proposed method, with a fixed $\bw$, if we only optimize $\bT$ with SGD, the magnitude of $\bT$ will grow unbounded, and this will lead to near-perfect training accuracy, and obviously,  overfitting. Therefore, a regularizer over $\bT$ should lead to better performance. We have tested different types of regularizations of $\bT$ such as the Frobenius norm, and the $\ell_1$ norm. Interestingly, such a regularization could only marginally improve the performance. It appears that early stopping in the alternating minimization framework provides reasonable regularization in practice on the tested datasets. 

\subsection{CNM as Neural Networks}
One can view the proposed CNM framework from a different angle. If we ignore the original motivation of the work \emph{i.e.}, kernel approximation via Random Fourier Features, the proposed method can be seen as a \emph{shallow} neural network with one hidden layer, with $\cos(\cdot)$ as the activation function, and the SVM objective. 
It is interesting to note that such a ``two-layer neural network'', which simulates certain shift-invariant kernels, leads to very good classification performance as shown in the experimental section. 
Under the neural network view, one can also use back-propagation as the optimization method, similar to the proposed  alternating SGD, or use other types of activation functions such as the sigmoid, and ReLU functions. However the ``network'' then will no longer correspond to a shift-invariant kernel. 

\section{Experiments}
\label{sec:exp}

\begin{table}[t]
\caption{8 UCI datasets used in the experiments}
\begin{center}
\small
\begin{tabular}{|c|c|c|c|c|}
\hline Dataset & Number of Training  & Number of Testing & Dimensionality  \\ 
\hline   \texttt{USPS}    & 7,291 & 2,007 &256 \\ 
\hline   \texttt{BANANA} & 1,000  & 4,300 &  2 \\ 
\hline   \texttt{MNIST} & 60,000 & 10,000 & 784\\ 
\hline   \texttt{CIFAR}  & 50,000 & 10,000 & 400\\ 
\hline   \texttt{FOREST} & 522,910 & 58,102& 54\\ 
\hline   \texttt{LETTER} & 12,000 & 6,000& 16\\ 
\hline   \texttt{MAGIC04} & 14,226 & 4,795 & 10\\ 
\hline   \texttt{IJCNN} & 49,990 & 91,701 & 22\\ 
\hline 
\end{tabular}
\end{center}
\label{table:dataset}
\end{table}

We conduct experiments using 8 UCI datasets summarized in Table \ref{table:dataset}. 
The size of the mini batches in the optimization are empirically set as 500. The number of SGD steps in optimizing $\bT$ and $\bw$ is set as 100. We find that satisfactory classification accuracy can be achieved within a few hundred iterations.  
The bandwidth of the RBF kernel in classification experiments, and the kernel approximation experiments is set to be $\gamma = 2/\sigma^2$, where $\sigma$ is the average distance to the 50th nearest neighbor estimated from 1,000 samples of the dataset. Further fine tuning of $\gamma$ may lead to even better performance.

\subsection{CNM for Classification}

\begin{figure*}[t]
\centering
\subfigure[\texttt{MAGIC04}]
{\includegraphics[width = 4cm]{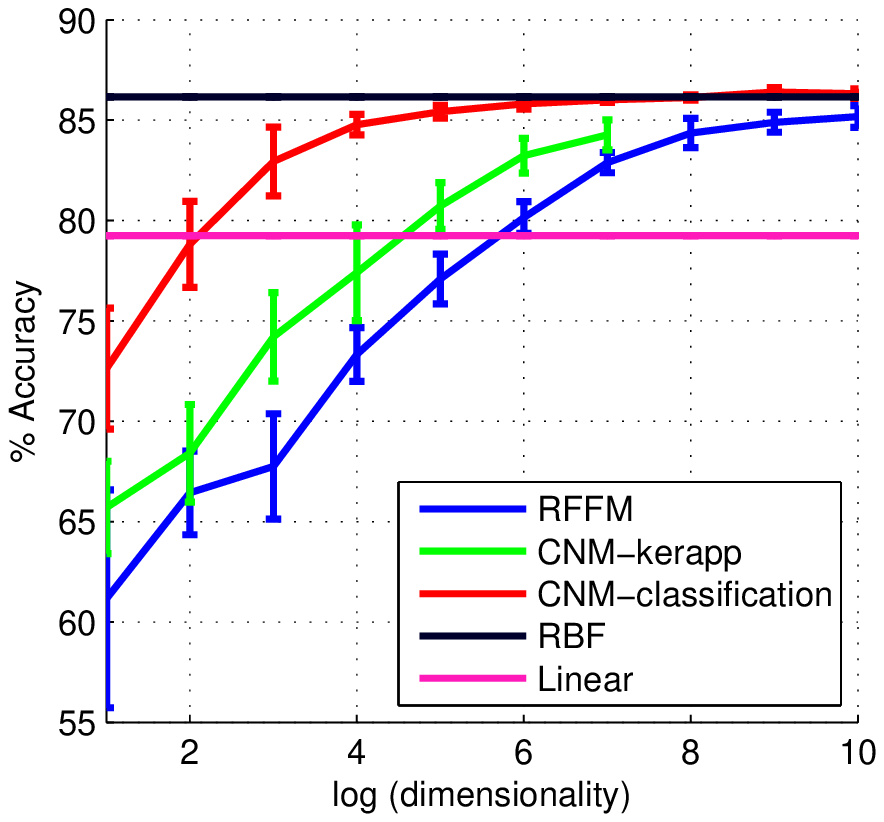}}
\subfigure[\texttt{MNIST}]
{\includegraphics[width = 4cm]{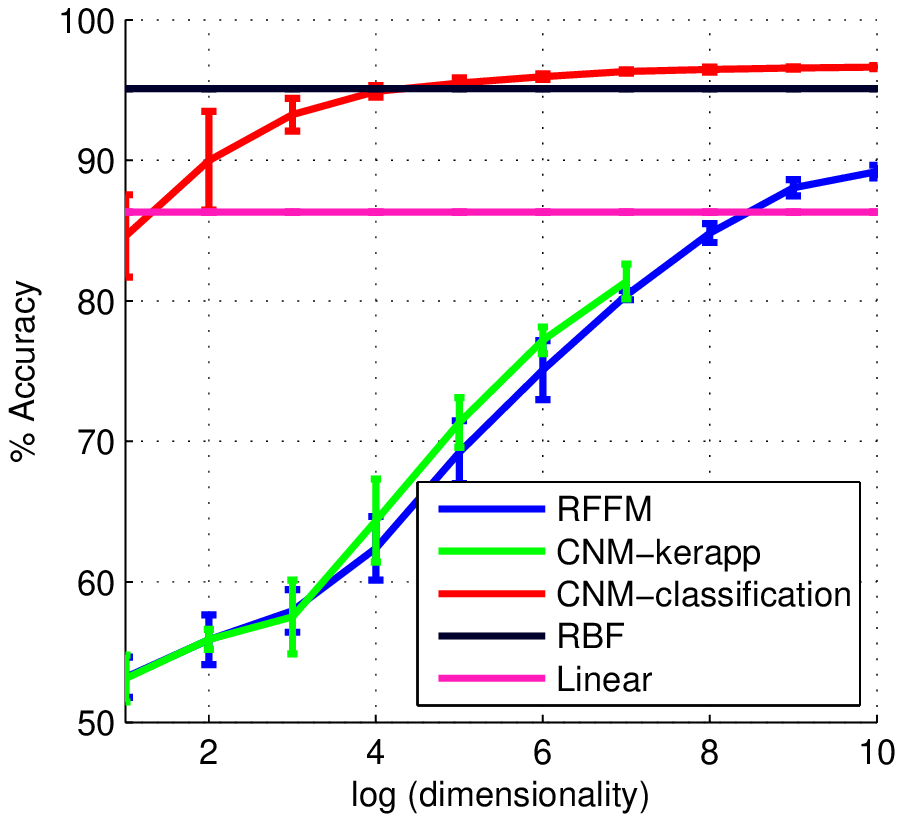}}
\subfigure[\texttt{USPS}]
{\includegraphics[width = 4cm]{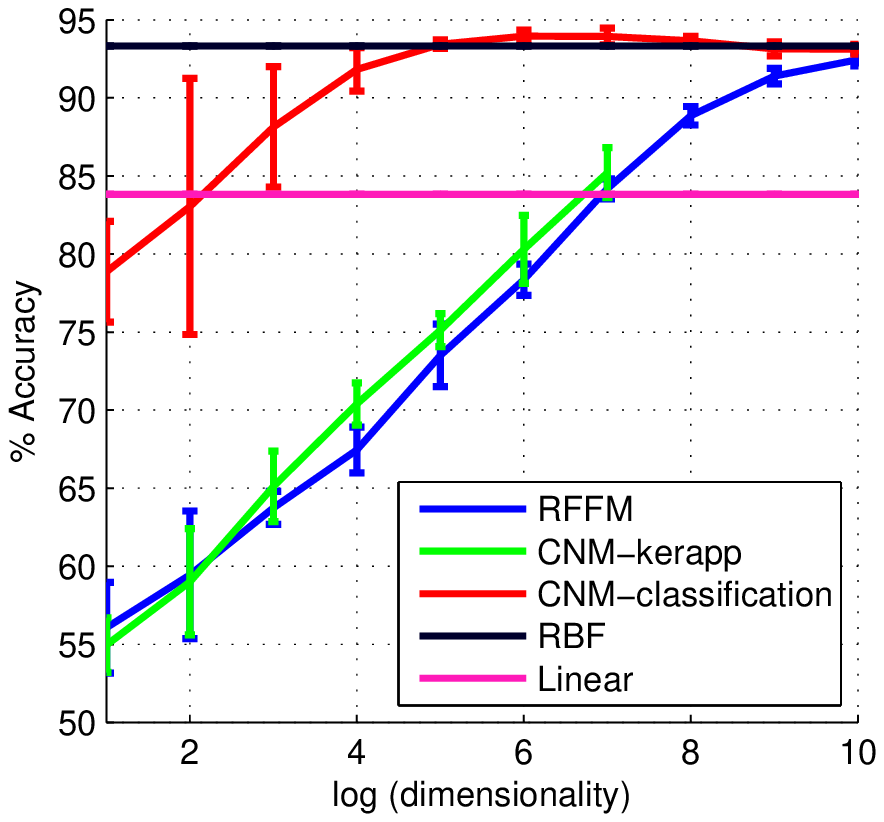}}
\subfigure[\texttt{BANANA}]
{\includegraphics[width = 4cm]{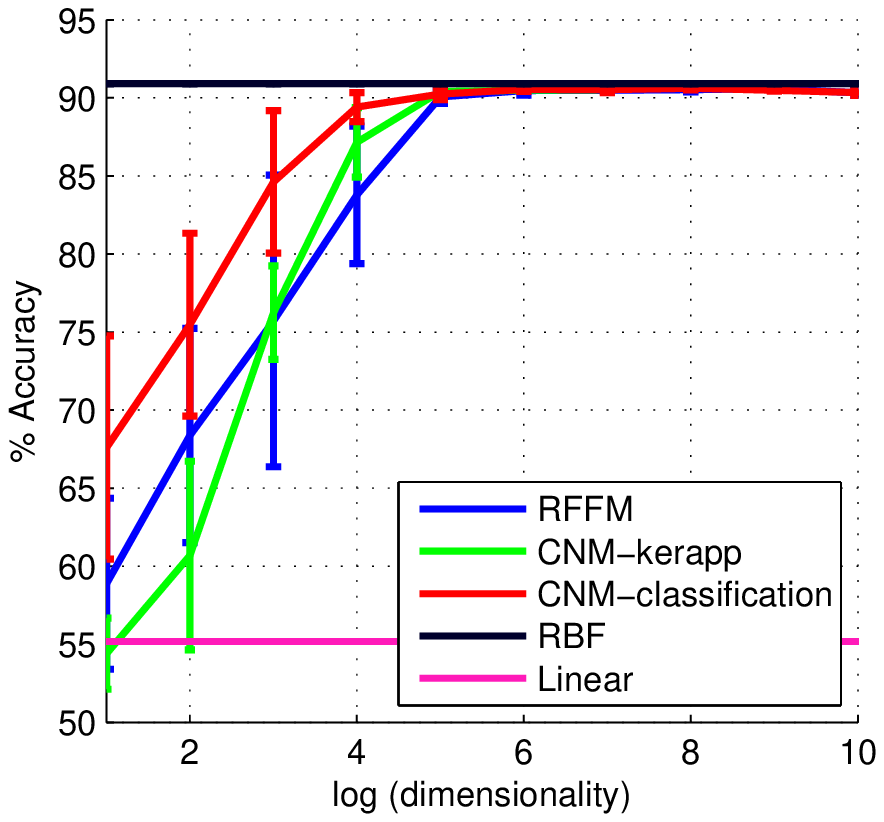}}
\subfigure[\texttt{FOREST}]
{\includegraphics[width = 4cm]{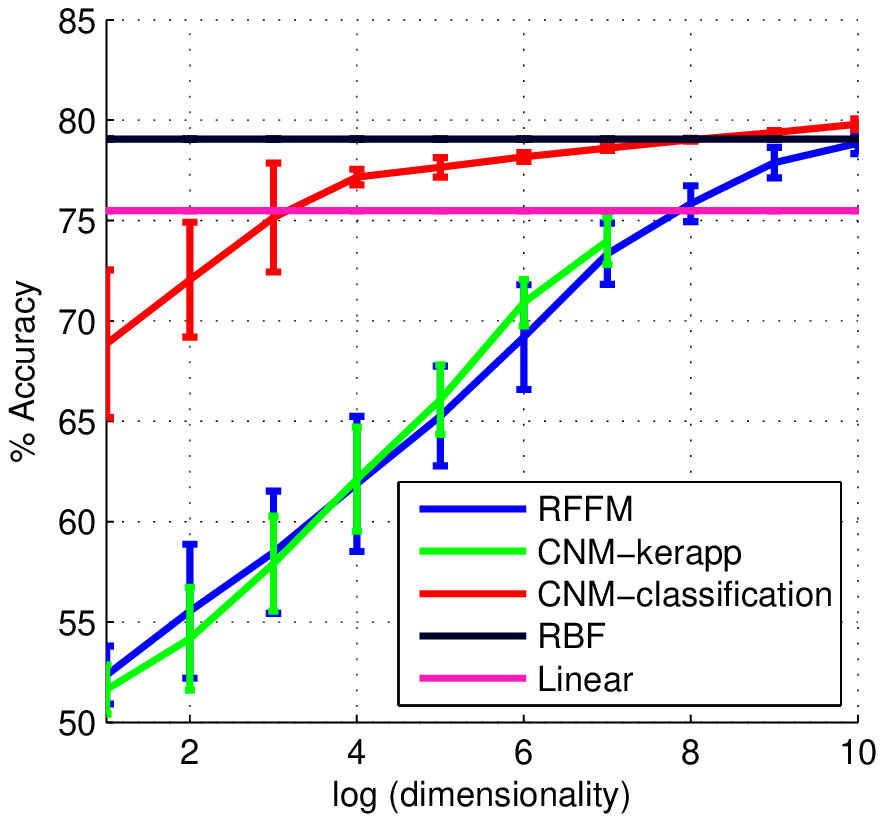}}
\subfigure[\texttt{CIFAR}]
{\includegraphics[width = 4cm]{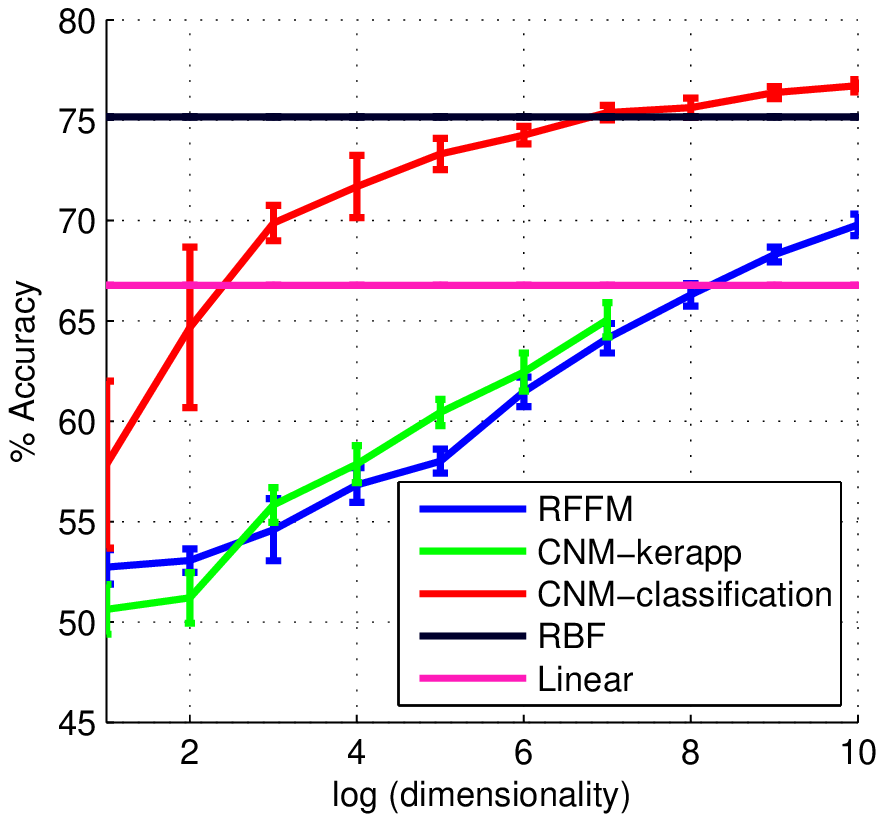}}
\subfigure[\texttt{LETTER}]
{\includegraphics[width = 4cm]{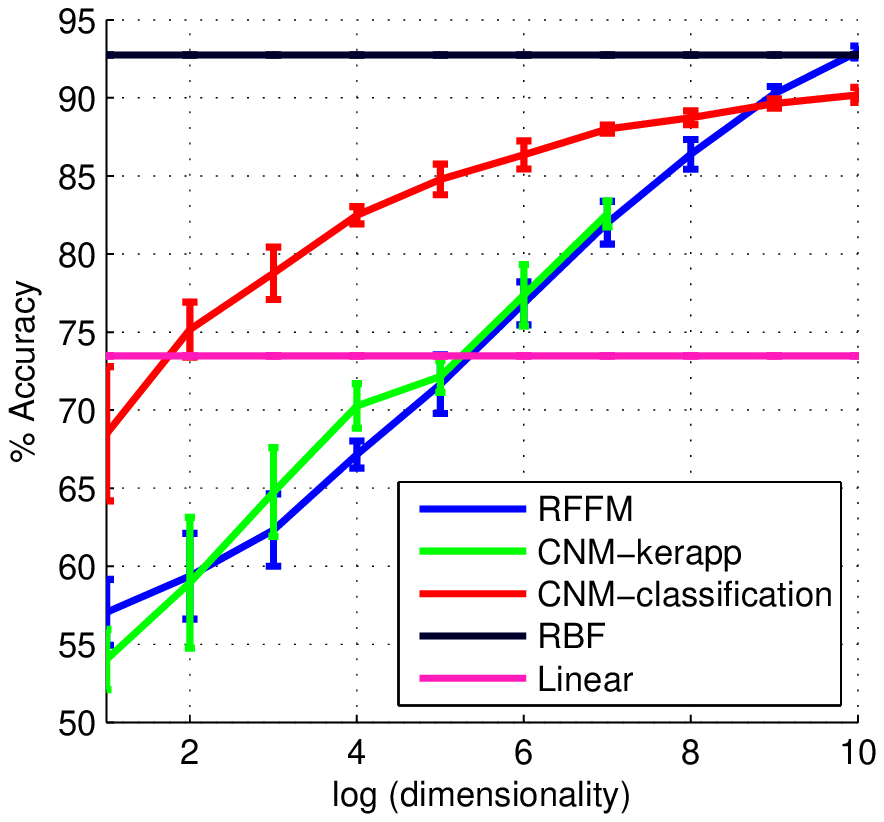}}
\subfigure[\texttt{IJCNN}]
{\includegraphics[width = 4cm]{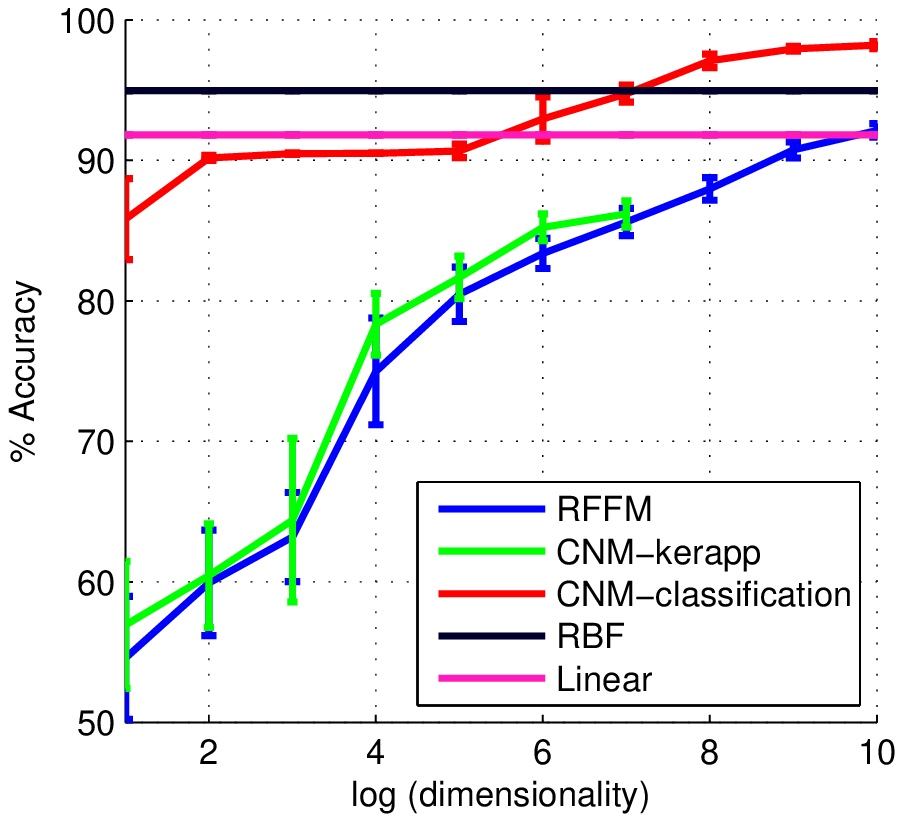}}
\caption{Compact Nonlinear Map (CNM) for classification. RFFM: Random Fourier Feature Map based on RBF kernel. CNM-kerapp: CNM for kernel approximation (Section \ref{sec:kernel_app}). CNM-classification: CNM for classification (Section \ref{sec:cnm}). RBF: RBF kernel SVM. Linear: linear SVM based on the original feature.}
\label{fig:acc}
\end{figure*}

Figure \ref{fig:acc} shows the classification accuracies. 
CNM-classification is the proposed method. We compare it with three baselines: linear SVM based on the original features (Linear), kernel SVM based on RBF (RBF), and the Random Fourier Feature method (RFFM). 
As shown in the figures, all the datasets are not linearly separable, as the RBF SVM performance is much better than the linear SVM performance. 

\begin{itemize}
\item For all the datasets, CNM is much more compact than the Random Fourier Feature to achieve the same classification accuracy. For example, on the USPS dataset, to get 90\% accuracy, the dimensionality of CNM is 8, compared to 512 of RFFM, a 60x improvement. 
\item As the dimensionality $k$ grows, accuracies of both the RFFM and CNM improve, with the RFFM approaching the RBF performance. In a few cases, the CNM performance can be even higher than the RBF performance. This is due to the fact that CNM is ``approximating'' an optimal kernel, which could be better than the fixed RBF kernel. 
\end{itemize}

\subsection{CNM for Kernel Approximation}
We conduct experiments on using the CNM framework to approximate a known kernel (Section \ref{sec:kernel_app}). The kernel approximation performance (measured by MSE) is shown in Figure \ref{fig:ker_app}. 
CNM is computed with dimensionality up to 128.
For all the datasets, CNM achieves more compact kernel approximations compared to the Random Fourier Features. 
We further use such features in the classification task. The performance is shown as the green curve (CNM-kerapp) in Figure \ref{fig:acc}. Although CNM-kerapp has lower MSE in kernel approximation than RFFM, its accuracy is only comparable or marginally better than RFFM. This verifies the fact that better kernel approximation may not necessarily lead to better classification. 

\begin{figure*}[t]
\centering
\subfigure[\texttt{MAGIC04}]
{\includegraphics[width = 4cm]{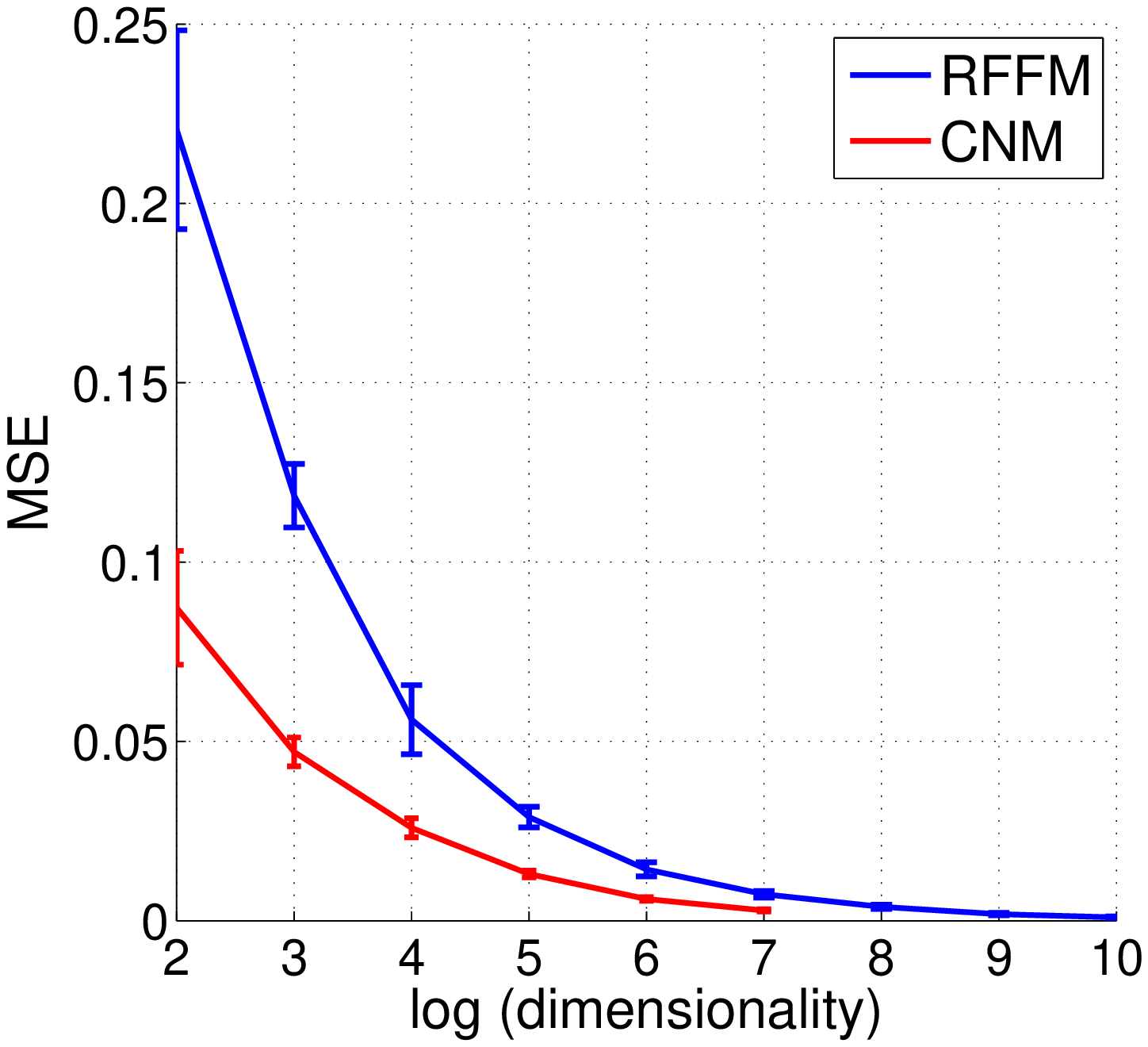}}
\subfigure[\texttt{MNIST}]
{\includegraphics[width = 4cm]{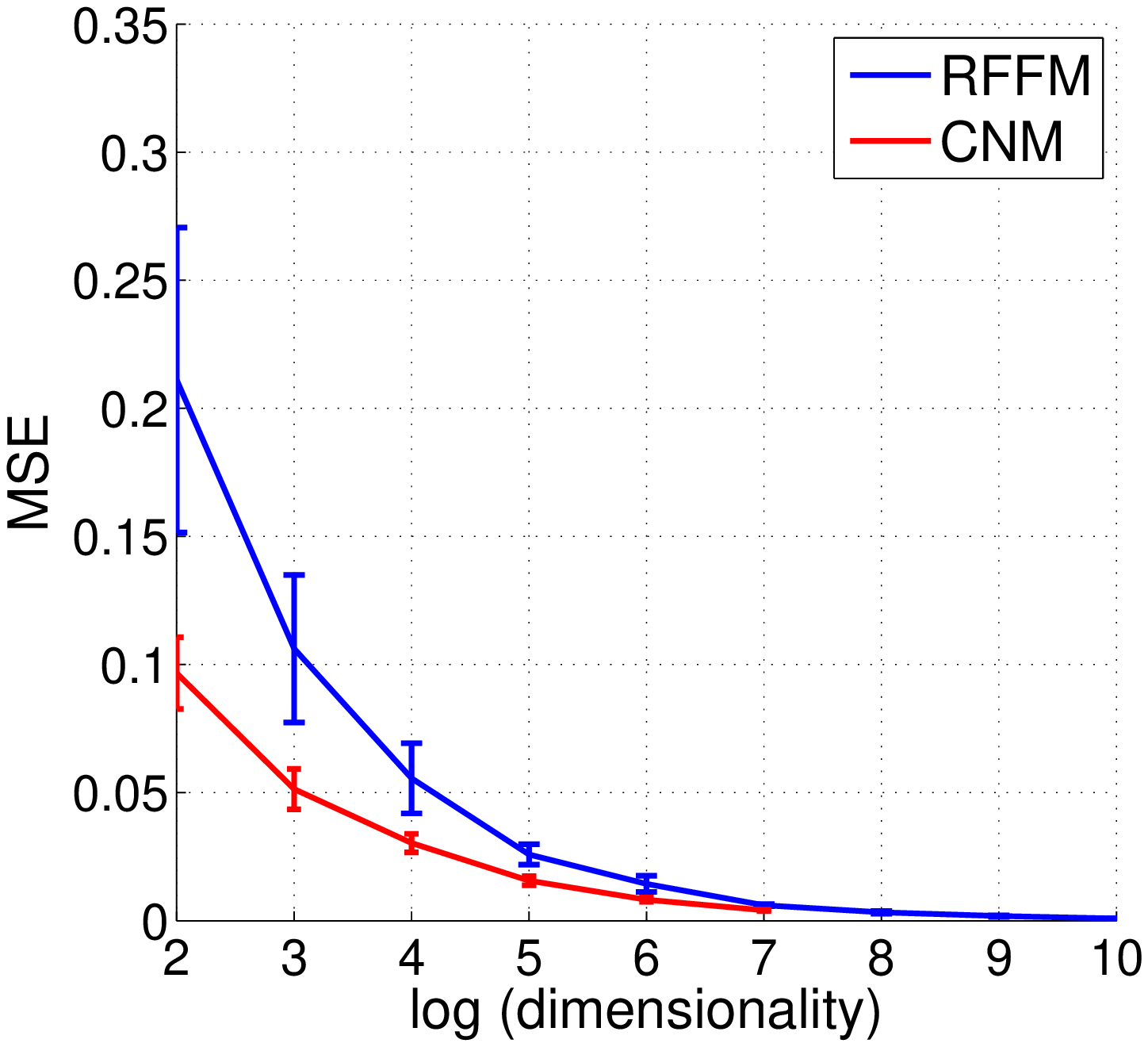}}
\subfigure[\texttt{USPS}]
{\includegraphics[width = 4cm]{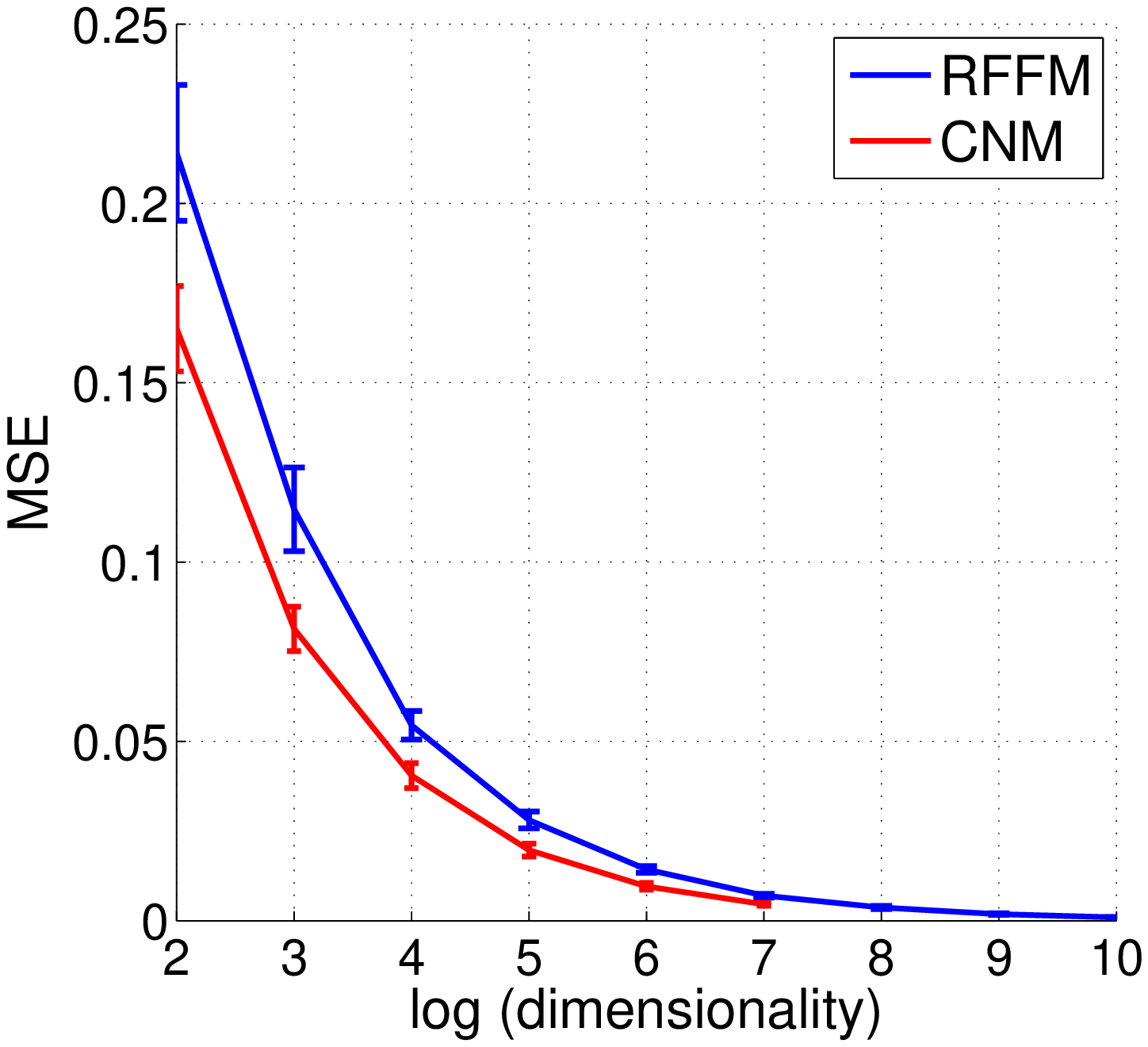}}
\subfigure[\texttt{BANANA}]
{\includegraphics[width = 4cm]{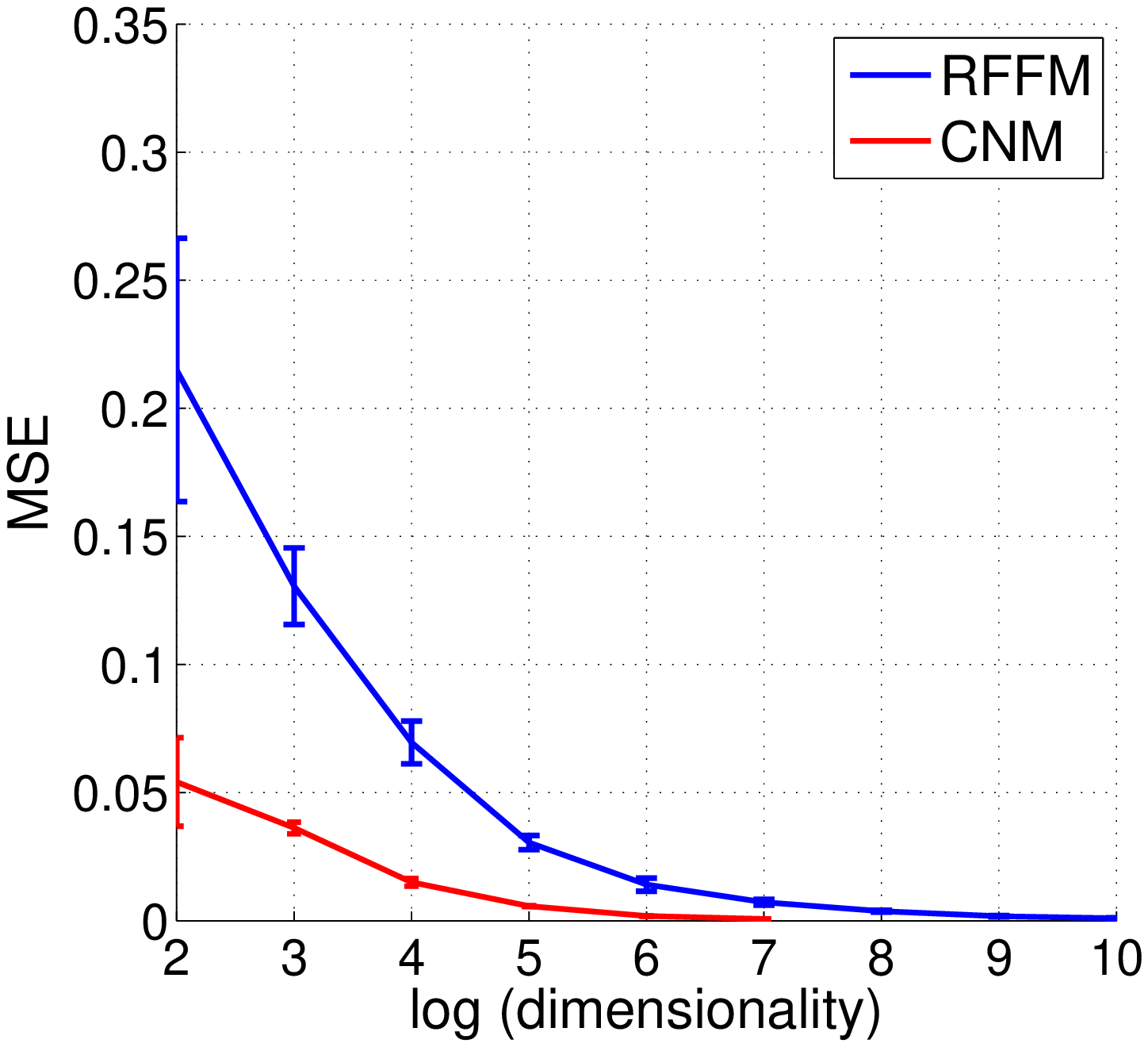}}
\subfigure[\texttt{FOREST}]
{\includegraphics[width = 4cm]{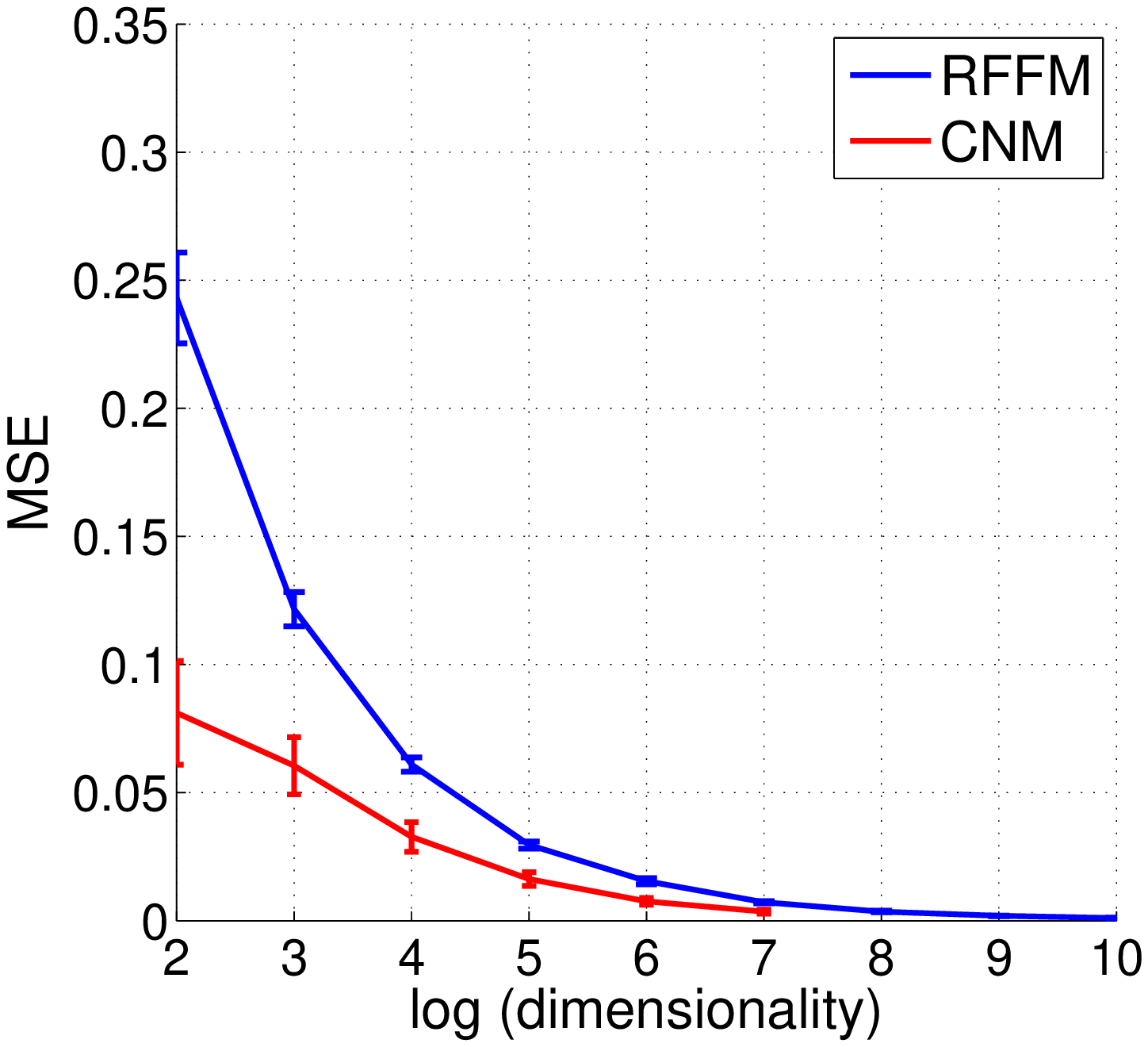}}
\subfigure[\texttt{CIFAR}]
{\includegraphics[width = 4cm]{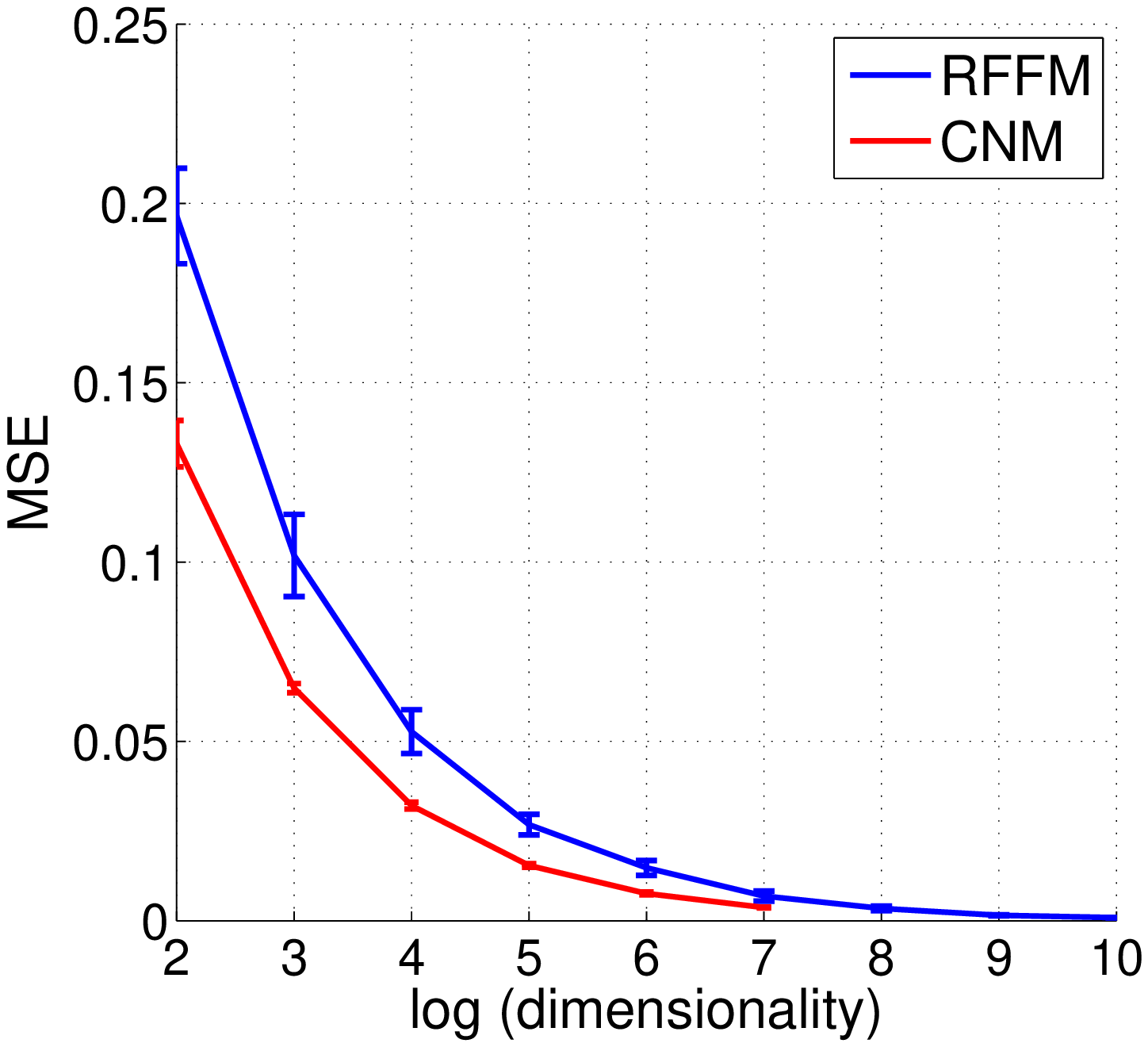}}
\subfigure[\texttt{LETTER}]
{\includegraphics[width = 4cm]{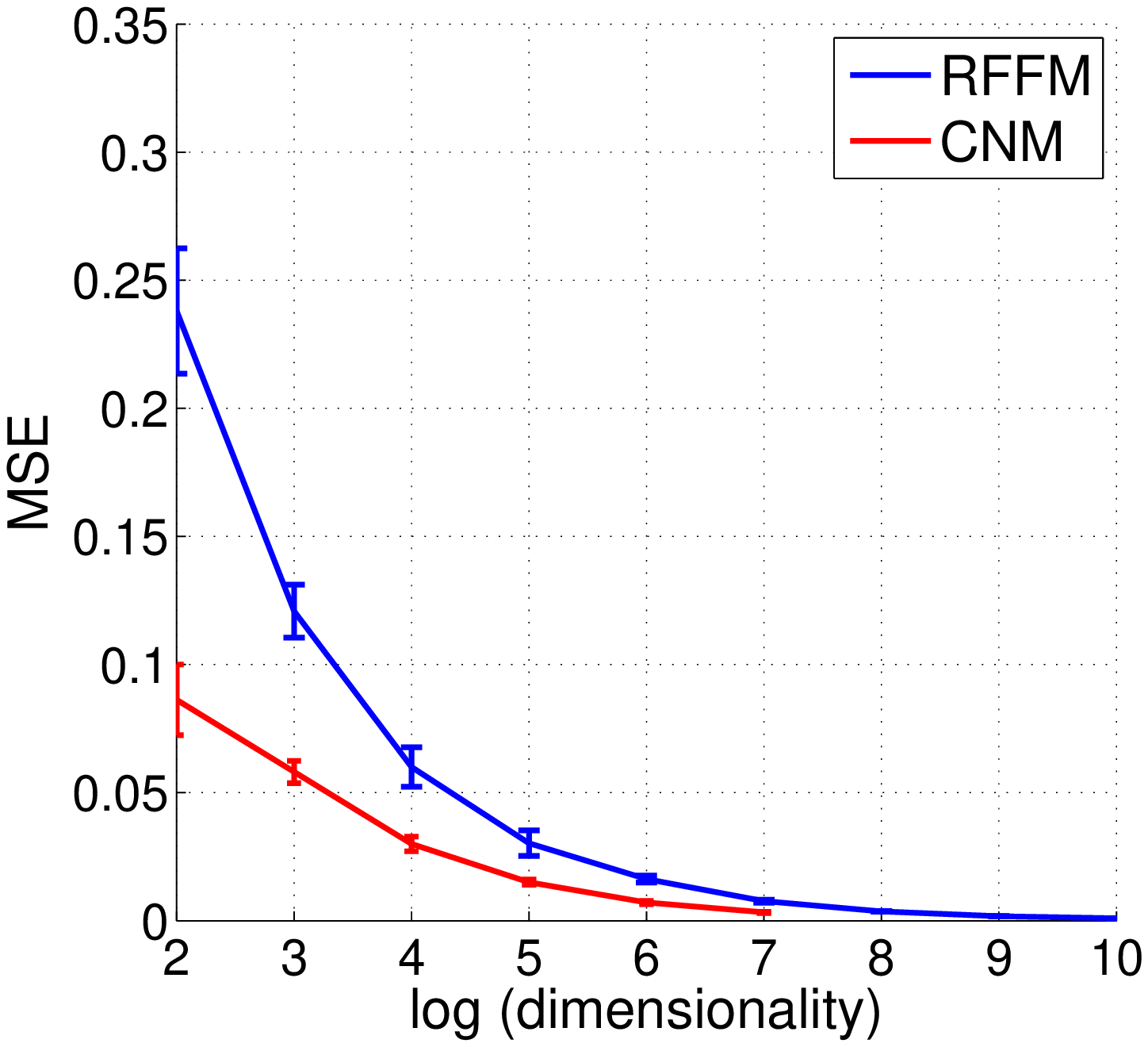}}
\subfigure[\texttt{IJCNN}]
{\includegraphics[width = 4cm]{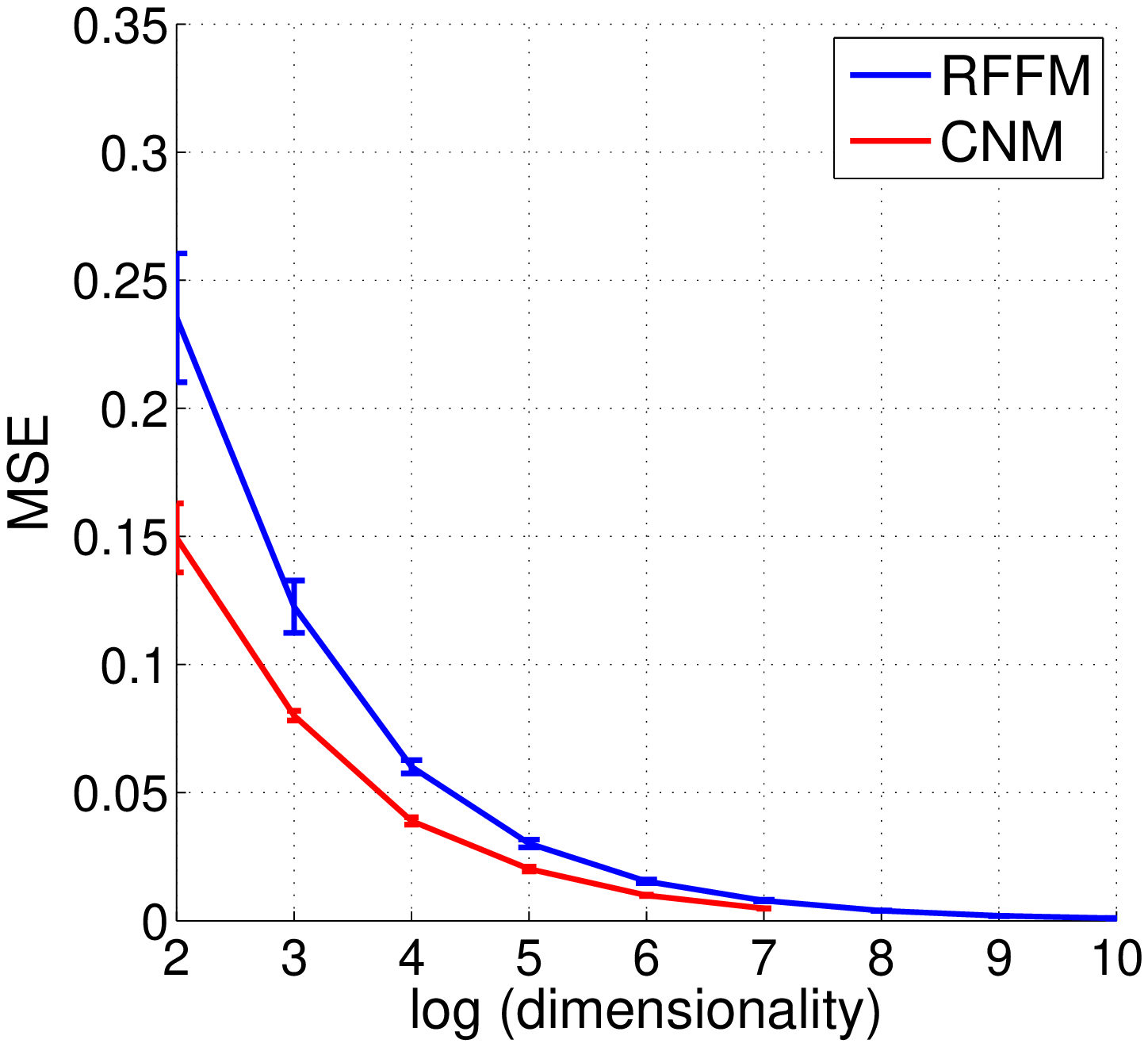}}
\caption{Compact Nonlinear Map (CNM) for kernel approximation.
RFFM: Random Fourier Feature based on RBF kernel. CNM-kerapp: CNM for kernel approximation (Section \ref{sec:kernel_app}).}
\label{fig:ker_app}
\end{figure*}

\section{An Extension: Circulant Nonlinear Maps}
\label{sec:circulant_ker}

Kernel approximation with nonlinear maps comes with an advantage that 
SVM can be trained in $\mathcal{O}(N)$, and evaluated in $\mathcal{O}(k)$ time, leading to scalable learning and inference. 
In this paper, we have presented CNM where the projection matrix of the Random Fourier Features is optimized to achieve high classification performance. 
For $d$-dimensional inputs and $k$-dimensional nonlinear maps, the computational and space complexities of both CNM and RFFM are $\mathcal{O}(kd)$. CNM comes with the advantage that $k$ can be much  smaller than that for RFFM to achieve a similar performance. 
One observation from Section \ref{sec:exp} is that though CNM can lead to much more compact maps, it still has better performance when higher-dimensional maps are used. In many situations, it is required that the number of nonlinear map $k$ is comparable to the feature dimension $d$. This will lead to both space and computation computation complexity $\mathcal{O}(d^2)$, which is not suitable for high-dimensional datasets. One natural question to ask is whether it is possible to further improve the scalability in terms of the input dimension $d$.


Structured matrices have been used in the past to simulate a fully randomized matrix in many machine learning settings, including dimensionality reduction \cite{vybiral2011variant, hinrichs2011johnson, ailon2006approximate}, binary embedding \cite{yu2014circulant}, and deep neural networks \cite{cheng2015fast}. 
In addition, the fast Johnson-Lindenstrauss type transformations can also be used in speeding up the Random Fourier Features in kernel approximation \cite{le2013fastfood}, and locality sensitive hashing \cite{dasgupta2011fast}. 
This comes with an advantage that linear projection with a suitably designed structured matrix can be more more space and time efficient. 
In this section, we show that by imposing the \emph{circulant structure} on the projection matrix, one can achieve similar kernel approximation performance compared to the fully randomized matrix. The proposed approach reduces the computational complexity to $\mathcal{O}(k \log d)$, and the space complexity to $\mathcal{O}(k)$, when $k \geq d$.

\subsection{Circulant Nonlinear Maps}
A circulant matrix $\mathbf{R} \in \mathbb{R}^{d \times d}$ is a matrix defined by a vector $\mathbf{r}  = (r_0, r_1, \cdots,  r_{d-1})$:

\begin{small}
\begin{align}
\mathbf{R} = \circR(\mathbf{r}) :=
\begin{bmatrix}
r_0      & r_{d - 1} & \dots  & r_{2} & r_{1}  \\
r_{1}    & r_{0}     & r_{d-1} &         & r_{2}  \\
\vdots   & r_{1}& r_0    & \ddots  & \vdots   \\
r_{d-2}  &        & \ddots & \ddots  & r_{d-1}   \\
r_{d-1}  & r_{d-2} & \dots  & r_{1} & r_{0}
\end{bmatrix}.
\label{eq:cir}
\end{align}
\end{small}

Let $\mathbf{D}$ be a diagonal matrix with each diagonal entry being a Bernoulli  variable ($\pm 1$ with probability 1/2).
For $\mathbf{x} \in \mathbb{R}^d$, its $d$-dimensional circulant nonlinear map is defined as:
\begin{align}
Z(\mathbf{x}) = \cos (\mathbf{R} \mathbf{D} \mathbf{x}), \quad \mathbf{R} = \circR(\mathbf{r}).
\label{eq:def_circulant_ker}
\end{align}

The diagonal matrix $\mathbf{D}$ is required in order to improve the capacity when using a circulant matrix for both binary embedding \cite{yu2014circulant} and dimensionality reduction \cite{vybiral2011variant}.  Since multiplication with a Bernoulli  random diagonal matrix corresponds to random sign flipping of each element of vector $\mathbf{x}$, this can be done as a pre-processing step. To simplify the notation, we omit this matrix in the following discussion.

A circulant matrix has the space complexity of $\mathcal{O}(d)$ . The other advantage of using the circulant projection is that the Fast Fourier Transform (FFT) can be used to speed up the computation. Denote $\circledast$ as the operator of a circulant convolution. Based on the definition of a circulant matrix,
\begin{align}
\mathbf{R} \mathbf{x} = \br \circledast \mathbf{x}.
\end{align}
The convolution above can be computed more efficiently in the Fourier domain, using the Discrete Fourier Transform (DFT), for which a fast algorithm (FFT) is available. 
\begin{align}
Z (\mathbf{x}) = \phi \left( \mathcal{F}^{-1} ( \mathcal{F}({\br}) \circ \mathcal{F}(\mathbf{x})) \right),
\end{align}
where $\circ$ denotes the element-wise product. $\mathcal{F}(\cdot)$ is the operator of DFT, and $\mathcal{F}^{-1}(\cdot)$ is the operator of inverse DFT (IDFT).
As DFT and IDFT can be efficiently computed in $\mathcal{O}(d \log{d})$ time  with FFT \cite{oppenheim1999discrete}, the proposed approach has time complexity $\mathcal{O}(d \log{d})$. Note that the circulant matrix is never explicitly computed or stored. The circulant projections are always performed by using FFT.

What we described above assumed a circulant nonlinear map with $k = d$.
When $k < d$, we can still use the circulant matrix $\bR \in \mathbb{R}^{d \times d}$ with $d$ parameters, but the output is set to be the first $k$ elements in Equation \ref{eq:def_circulant_ker}. When $k > d$, we use multiple circulant projections, and concatenate their outputs. This gives the computational complexity $\mathcal{O}(k \log d)$, and space complexity $\mathcal{O}(k)$. Note that the DFT of the feature vector can be reused in this case.

\subsection{Randomized Circulant Nonlinear Maps}

Similar to the Random Fourier Features, one can generate the parameters of the circulant projection (\emph{i.e.}, the elements of vector $\mathbf{r}$ in Equation \ref{eq:def_circulant_ker}) via random sampling from a Gaussian distribution. We term such a method randomized circulant nonlinear maps. Figure \ref{fig:circulant_kernel_app} shows the kernel approximation MSE of the randomized circulant nonlinear maps and compares it with the Random Fourier Features. Although with much better computational and space complexity, it is interesting that the circulant nonlinear map can achieve almost identical MSE compared to the Random Fourier Features. 

\begin{figure*}
\centering
\subfigure[\texttt{USPS}]
{\includegraphics[width = 5cm]{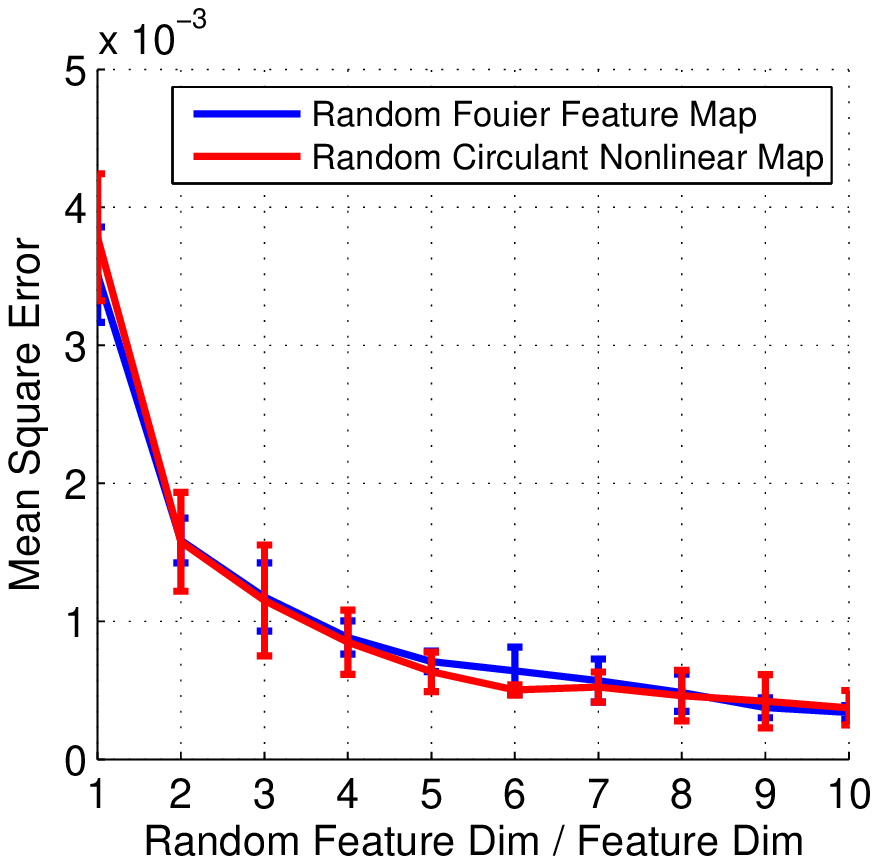}}
\subfigure[\texttt{CIFAR}]
{\includegraphics[width = 5cm]{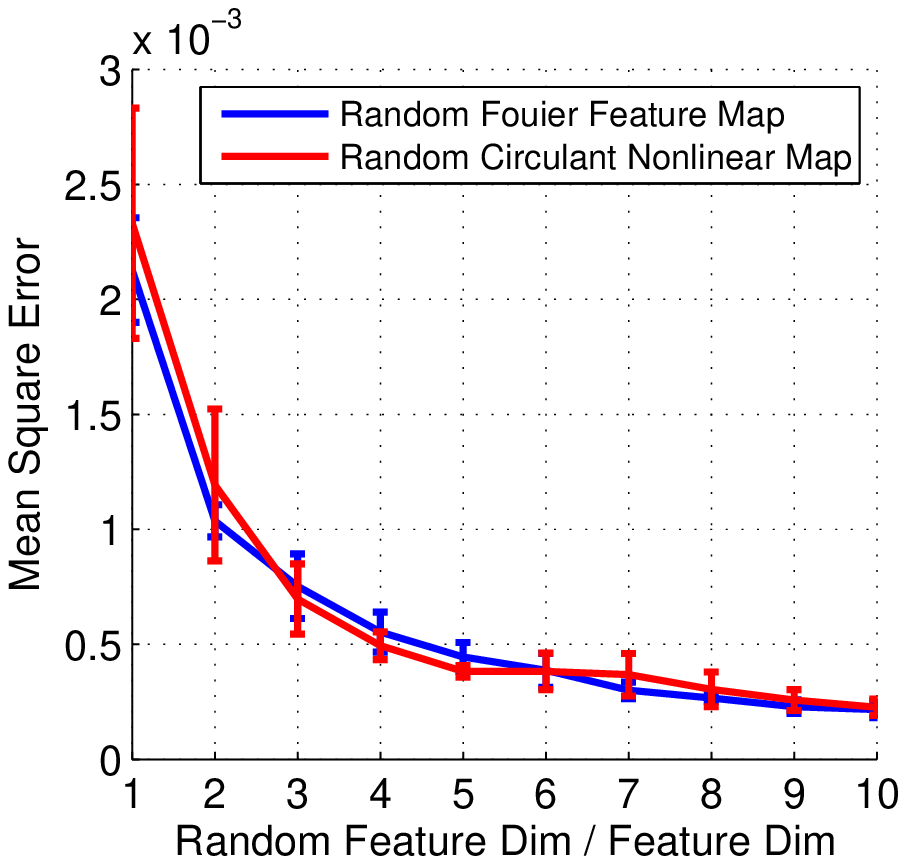}}
\subfigure[\texttt{MNIST}]
{\includegraphics[width = 5cm]{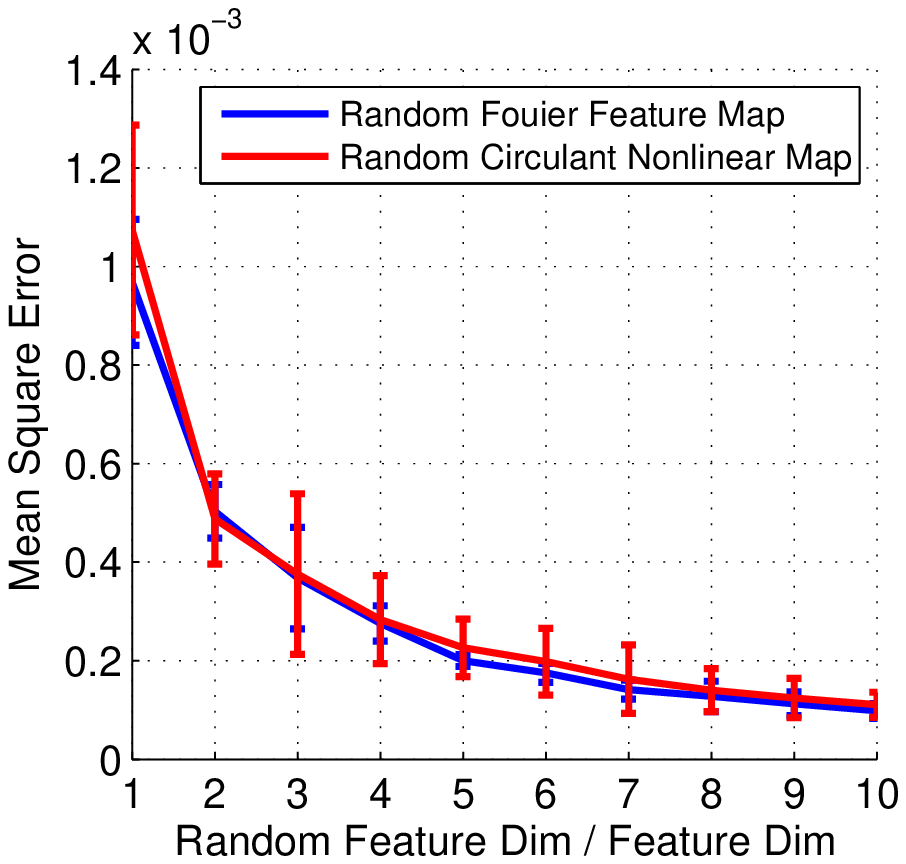}}
\caption{MSE of Random Fourier Feature, and randomized circulant nonlinear map.}
\label{fig:circulant_kernel_app}
\end{figure*}

\subsection{Optimized Circulant Nonlinear Maps}
Following the CNM framework, one can optimize the parameters in the projection matrix to improve the performance using alternating minimization procedure with the classification objective. The step to optimize classifier parameters $\bw$ is the same as described in section \ref{sec:alt_min}. The parameters of the projection are now given by circulant matrix $\bR$. Thus, the step of optimizing $\bR$  requires computing the gradient with respect to each element of vector $\mathbf{r}$ as:
\begin{equation}
\frac{\partial \bw^T \cos(\bR \bx)}{\partial r_i} = - \bw^T \left( \sin(\bR \bx) \circ s_{\rightarrow i} (\bx) \right) = - s_{\rightarrow i} (\bx)^T (\bw \circ \sin(\bR \bx)),
\label{eq:gradient}
\end{equation}
where $s_{\rightarrow i} (\cdot) : \mathbb{R}^d \rightarrow \mathbb{R}^d$,    circularly (downwards) shifts the vector $\bx$ by one element. 
Therefore,
\begin{align}
\nabla_{\br} (\bw^T \cos(\bR \bx)) &= -[s_{\rightarrow 0} (\bx), s_{\rightarrow 1} (\bx)
, \cdots, s_{\rightarrow (d-1)} (\bx)]^T (\bw \circ \sin (\bR \bx)) \\
&= - \circR(s_{\rightarrow 1} (\text{rev} (\bx) ))  (\bw \circ \sin(\bR \bx))  \nonumber \\
&= - s_{\rightarrow 1} (\text{rev} (\bx) ) \circledast (\bw \circ \sin(\br \circledast \bx)),  \nonumber
\end{align}
where $\text{rev} (\bx) = (x_{d-1}, x_{d-2}, \dots, x_0), \quad 
s_{\rightarrow 1} (\text{rev} (\bx) ) = 
(x_0, x_{d-1}, x_{d-2}, \cdots, x_1)$.

The above uses the same trick of converting the circulant matrix multiplication to circulant convolution. Therefore, computing the gradient of $\br$ takes only $\mathcal{O}(d\log d)$ time. The classification accuracy on three datasets with relatively large feature dimensions are shown in Table \ref{table:circulant_acc}. The randomized circulant nonlinear maps give similar performance to that from the Random Fourier Features but with much less storage and computation time. Optimization of circulant matrices tend to further improve the performance.

\begin{table}
\centering
\begin{tabular}{  l | c | c | c  }
    \hline
    Dataset (dimensionality $k$) & Random Fourier Feature & Circulant-random & Circulant-optimized \\ \hline
    \texttt{USPS} ($d$)  & $89.05 \pm 0.65$ & $89.40 \pm 1.02$  & \bm{$91.96 \pm 0.45$} \\ \hline
    \texttt{USPS} ($2d$) & $91.90 \pm 0.29$ & $91.87 \pm 0.11$  & \bm{$93.08 \pm 0.96$} \\ \hline
    \texttt{MNIST} ($d$) & $91.33 \pm 0.05$ & $91.01 \pm 0.03$ & \bm{$92.73 \pm 0.21$} \\ \hline
    \texttt{MNIST} ($2d$)& $92.95 \pm 0.42$ &  $93.22 \pm 0.30$ & \bm{$94.11 \pm 0.24$} \\ \hline
    \texttt{CIFAR} ($d$)  & $69.14 \pm 0.64$ &  $65.21 \pm 0.18$ & \bm{$71.17 \pm 0.68$}  \\ \hline
    \texttt{CIFAR} ($2d$) & \bm{$71.15 \pm 0.28$} & $68.56 \pm 0.70$  & $71.11 \pm 0.46$ \\ \hline
\end{tabular}
\caption{Classification accuracy (\%) using circulant nonlinear maps. The randomized circulant nonlinear maps have similar performance as of the Random Fourier Features but with significantly reduced storage and computation time. Optimization of circulant matrices tend to further improve the performance. }
\label{table:circulant_acc}
\end{table}

\section{Conclusion}
We have presented Compact Nonlinear Maps (CNM), which are motivated by the recent works on kernel approximation that allow very large-scale learning with kernels. This work shows that instead of using randomized feature maps,  learning the feature maps directly, even when restricted to shift-invariant kernel family, can lead to substantially compact maps with similar or better performance. The improved performance can be attributed mostly to  simultaneous learning of kernel approximation along with the classifier parameters. This framework can be seen as a shallow neural network with a specific nonlinearity (cosine) and provides a bridge between two seemingly unrelated streams of works. To make the proposed approach more scalable for high-dimensional data, we further introduced an extension, which imposes the circulant structure on the projection matrix. This improves the computation complexity from $\mathcal{O}(kd)$ to $\mathcal{O}(k \log d) $ and the space complexity from $\mathcal{O}(k d)$ to $\mathcal{O}(k)$, where $d$ is the input dimension, and $k$ is the output map dimension. In the future it will be interesting to explore if the complex data transforms captured by multiple layers of a deep neural network can be captured by learned nonlinear maps while remaining compact with good training and testing efficiency.

\vspace{+0.5cm}
\noindent\textbf{Acknowledgment.}
We would like to thank Weixin Li, David Simcha, Ruiqi Guo, and Krzysztof Choromanski for the helpful discussions. 

\bibliographystyle{plain}
\bibliography{design}

\end{document}